\documentclass[letterpaper, 10 pt, conference]{ieeeconf} 
\IEEEoverridecommandlockouts
\usepackage{amsmath,amssymb,amsfonts}
\usepackage{algorithmic}
\usepackage{graphicx}
\usepackage{textcomp}
\usepackage{xcolor}
\usepackage{comment}
\usepackage{url}
\usepackage{hyperref}
\usepackage{tikz}
\usepackage{booktabs}
\usepackage{mathtools}
\usepackage{cite}             
\usetikzlibrary{matrix}

\newtheorem{theorem}{Theorem}
\newtheorem{assumption}{Assumption}

\usepackage{etoolbox}
\makeatletter
\def\@IEEEtablestring{table}
\long\def\@makecaption#1#2{%
\ifx\@captype\@IEEEtablestring%
\noindent{\footnotesize #1.~~ #2\par}% Justified by default
\@IEEEtablecaptionsepspace%
\else
\@IEEEfigurecaptionsepspace%
\setbox\@tempboxa\hbox{\footnotesize #1.~~ #2}%
\ifdim \wd\@tempboxa >\hsize%
\setbox\@tempboxa\hbox{\footnotesize #1.~~ }%
\parbox[t]{\hsize}{\footnotesize \noindent\unhbox\@tempboxa#2}%
\else%
\ifcenterfigcaptions \hbox to\hsize{\footnotesize\hfil\box\@tempboxa\hfil}%
\else \hbox to\hsize{\footnotesize\box\@tempboxa\hfil}%
\fi\fi\fi}
\makeatother

\begin{document}
\title{\textbf{Environment-Aware Learning of Smooth GNSS \\ Covariance Dynamics for Autonomous Racing}\\
}

\author{Y. Deemo Chen$^{*,\dag}$, Arion Zimmermann$^{*,\dag}$, Thomas A. Berrueta$^{\dag}$, and Soon-Jo Chung$^{\dag}$%
\thanks{Accepted to IEEE ICRA 2026.}%
\thanks{$^{*}$Equal contribution.}%
\thanks{$^{\dag}$Department of Computing and Mathematical Sciences, California Institute of Technology, Pasadena, CA.}%
\thanks{Video: \url{https://youtu.be/fkOKFMRtwyk}}%
\thanks{This project was in part funded by Beyond Limits.}%
}

\IEEEaftertitletext{\vspace{-1\baselineskip}}  % Add before \maketitle
\maketitle

\begin{abstract}
Ensuring accurate and stable state estimation is a challenging task crucial to safety-critical domains such as high-speed autonomous racing, where measurement uncertainty must be both adaptive to the environment and temporally smooth for control. In this work, we develop a learning-based framework, LACE, capable of directly modeling the temporal dynamics of GNSS measurement covariance. We model the covariance evolution as an exponentially stable dynamical system where a deep neural network (DNN) learns to predict the system's process noise from environmental features through an attention mechanism. By using contraction-based stability and systematically imposing spectral constraints, we formally provide guarantees of exponential stability and smoothness for the resulting covariance dynamics. We validate our approach on an AV-24 autonomous racecar, demonstrating improved localization performance and smoother covariance estimates in challenging, GNSS-degraded environments. Our results highlight the promise of dynamically modeling the perceived uncertainty in state estimation problems that are tightly coupled with control sensitivity.
\end{abstract}

% \begin{IEEEkeywords}
% Covariance Estimation, Adaptive Filtering, Contraction Theory, Autonomous Racing.
% \end{IEEEkeywords}

\section{Introduction}
Accurate and reliable state estimation is a key component of safety-critical autonomous systems, where centimeter-level precision is often essential for safe navigation and operation~\cite{barfoot_state_2024}. Autonomous racing in competition circuits such as the Indy Autonomous Challenge (IAC) represents an extreme instance of this requirement, demanding not only high-accuracy localization but also exceptionally smooth state estimates to ensure stable vehicle control at speeds exceeding 150 mph~\cite{iac_racecar_nodate}. However, maintaining this performance is complicated by environmental factors that degrade sensor measurements. When a vehicle passes through Global Navigation Satellite System (GNSS)-degraded areas, such as under bridges or near buildings, multipath interference and poor satellite visibility can cause measurement noise to fluctuate dramatically, increasing from less than a meter to over 50 meters in less than a second.

In high-speed autonomous systems, standard estimation methods, such as the Extended Kalman Filter (EKF), face a dual challenge: the measurement noise covariance must be both adaptive with respect to the environment and temporally smooth and stable. The need for smoothness is driven by the downstream control system, where abrupt changes in perceived measurement uncertainty can destabilize the vehicle by triggering aggressive maneuvers. Simultaneously, the covariance must adapt to rapid changes in GNSS quality to prevent the filter from becoming dangerously over-confident in noisy data or overly cautious with clean data. Satisfying this dual requirement is particularly challenging for conventional covariance models. A constant covariance, while inherently smooth, is not adaptive and leads to poor performance at environmental extremes~\cite{akhlaghi_adaptive_2017}. Moreover, models based on receiver-provided metrics, such as Dilution of Precision (DOP), adapt to satellite geometry but can change erratically and, more importantly, cannot account for local effects like signal multipath~\cite{lohan_accuracy_2016}. As a result, standard approaches fail to consider the coupled dynamics between state estimation and control.

\begin{figure}[!t]
    \includegraphics[width=\columnwidth]{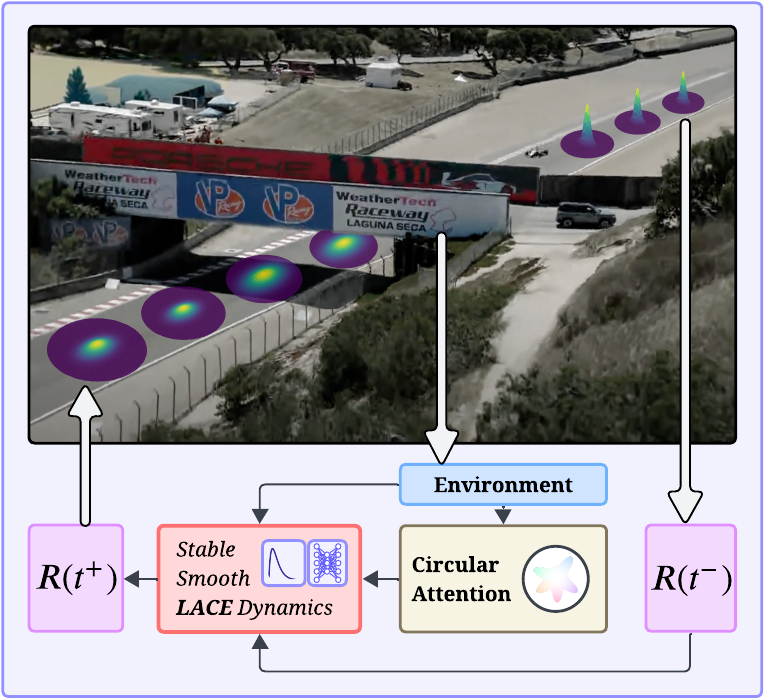}
    
    \vspace{0.5em} % space between the two figures
    \centering
    \includegraphics[width=0.995\columnwidth]{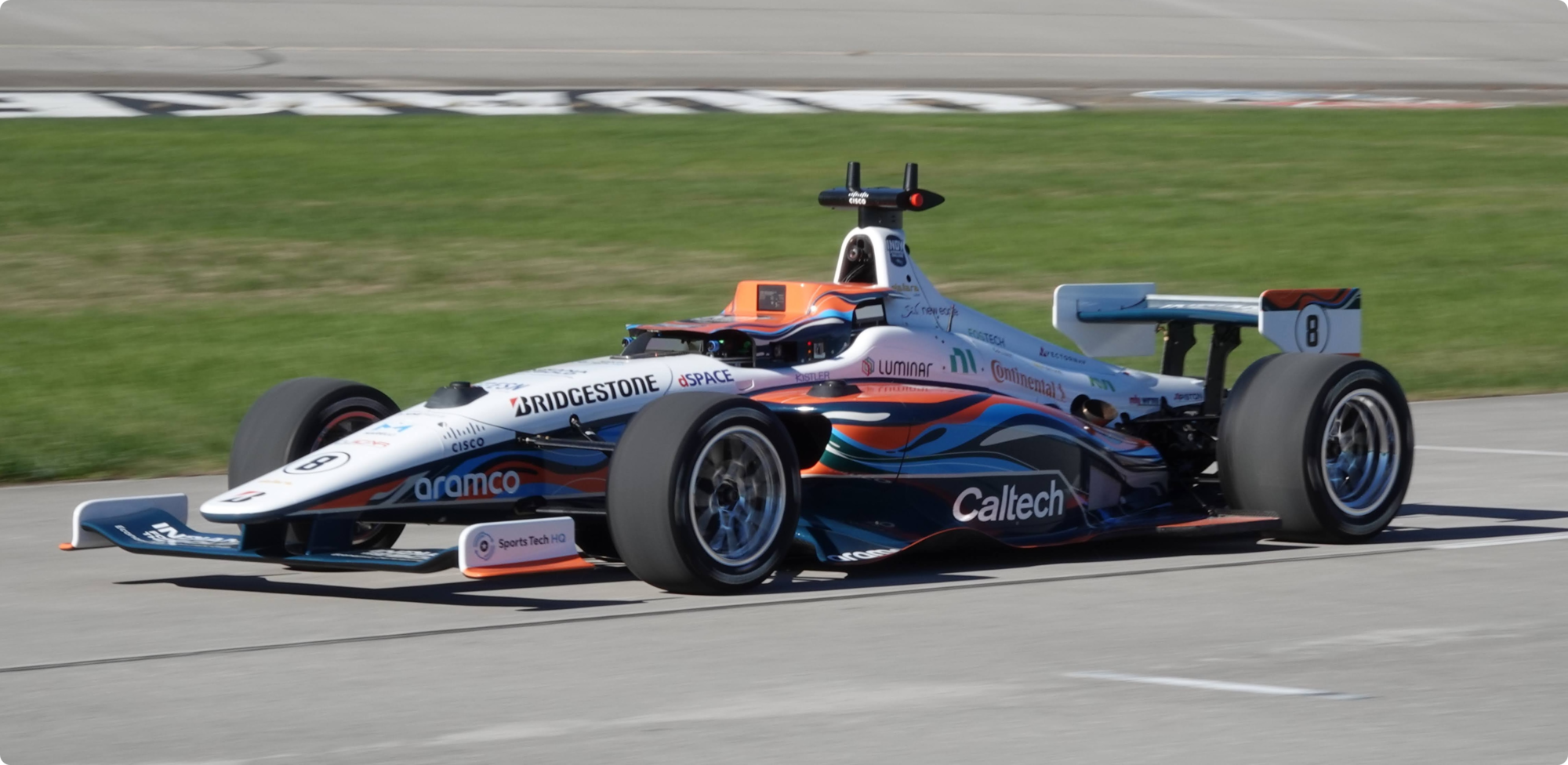}
    \caption{LACE applied to IAC racecars. \textbf{Top}: Illustration of the learned GNSS covariance evolution as the racecar passes under a concrete bridge at Laguna Seca Raceway. The structure degrades GNSS measurement quality, causing the covariance to broaden. For clarity, the covariance distribution is visualized in 2D. \textbf{Bottom}: Caltech IAC AV-24 racecar.}
    \label{fig:iac_cars}
\end{figure}

Recent methods address environmental adaptation and smoothness as isolated objectives without modeling covariance evolution as a dynamical system. For environmental awareness, DNNs have proven effective at learning an end-to-end mapping from raw sensor data to uncertainty, such as predicting instantaneous measurement covariance directly from sensor observables like images without needing hand-coded features~\cite{liu_deep_2018}. This has been extended to simultaneously estimate both process and measurement noise from temporal data streams using multitask learning architectures~\cite{wu_predicting_2021}. Other approaches use intermediate representations such as crowdsourced GNSS signal quality maps to predict measurement noise in urban areas~\cite{smolyakov_resilient_2020}. Some non-parametric framework predicts covariance by weighing empirical errors from nearby samples in a feature space~\cite{vega-brown_cello_2013}. This approach has since been applied to 3D point cloud registration~\cite{landry_cello-3d_2019} and has been extended with expectation-maximization to enable learning without ground truth state information~\cite{vega-brown_cello-em_2013}. While these methods capture complex sensor-to-uncertainty mappings, they output static covariance estimates without temporal dynamics. In other words, the learned models can produce discrete jumps in uncertainty that may destabilize downstream systems.

In contrast, techniques that prioritize temporal smoothness are often reactive and less capable of adapting to abrupt environmental shifts. Classical adaptive filtering, for example, adjusts noise parameters based on recent performance by analyzing its innovation sequence~\cite{mohamed_adaptive_1999}. Classical techniques rooted in innovation-based adaptive estimation operate on the principle that an optimal filter produces a white-noise innovation sequence and make use of this property to iteratively identify the process and measurement noise covariances~\cite{mehra_identification_1970}. Alternatively, methods based on multiple-model adaptive estimation run banks of parallel filters with different noise model assumptions and compute final state estimates as their weighted averages~\cite{magill_optimal_1965}. Smoothing techniques offer another path to stability by processing data in both forward and backward passes to produce optimal state trajectories, which is particularly effective during sensor outages~\cite{chiang_-line_2012}. However, these methods are either reactive or constrained to near-real-time post-processing, and cannot preemptively adapt to the rapid, environment-driven changes captured by predictive models.

To bridge the gap between purely predictive and reactive methods, some approaches explicitly model covariance as a regression problem. These methods construct a functional mapping from a feature space to a covariance matrix, using either parametric forms~\cite{hu_parametric_2015, hoff_covariance_2012} or non-parametric techniques like Gaussian Processes~\cite{kersting_most_2007}. While effective, these regression methods model the output of the uncertainty process but not the process itself. By framing each prediction independently, they neglect the inherent temporal correlation in how state uncertainty evolves. For example, sequential architectures may process temporal features~\cite{wu_predicting_2021}, but they still generate a series of static covariance estimates without explicitly modeling the transition between them.  This can lead to discontinuities that are problematic for control, a challenge that smoothing techniques address but only in post-processing~\cite{chiang_-line_2012}. In order to simultaneously achieve environmental adaptation and tunable smoothness there is a need for frameworks capable of learning the underlying covariance temporal dynamics.

In this work, we address this gap by introducing LACE (\emph{L}earning \emph{A}daptive \emph{C}ovariance \emph{E}volution), a learning-based framework that directly models the temporal dynamics of GNSS measurement covariance in an environmentally-aware fashion. We model the covariance evolution as an exponentially stable linearized time-varying dynamical system. On the one hand, a DNN learns to predict the system's process noise matrix, $Q(t)$, from environmental features. In particular, an attention mechanism produces spatial embeddings of the robot's state in order to capture sharp and highly localized environmental features, such as bridges or buildings. On the other hand, $A$ is constructed to be globally exponentially stable with spectral constraints, ensuring that the covariance evolution is both stable and smooth by design. Our contributions are then the following: 1) we propose a novel architecture that unifies environmental adaptation and tunable smoothness; 2) we prove the exponential stability and convergence of the covariance dynamics using contraction analysis~\cite{lohmiller1998contraction,tsukamoto_contraction_2021}; and 3) we validate our approach on an AV-24 autonomous racecar at the Laguna Seca Raceway, demonstrating improved performance in GNSS-degraded areas.

\section{Problem Definition}
 
Consider a robot with state $x(t) \in \mathbb{R}^m$, which includes the position vector $p(t) \in \mathbb{R}^3$. We define the GNSS measurement residual as $\epsilon(t) = p(t) - y(t)$, where $y(t)$ is the received GNSS position. Assuming a common coordinate frame, we model these residuals as realizations of a heteroscedastic Gaussian distribution, 
\begin{equation}
\label{eq:dist}
\epsilon(t) \sim \mathcal{N}(0, R(t)),
\end{equation} where the time-varying covariance matrix $R(t)$ is the quantity we aim to estimate from the robot's estimated state, $\hat{x}(t)$, and GNSS quality metrics, $g(t)$, such as DOP measures and satellite count.

% Our primary objective is to learn the covariance matrix $R(t)$ by minimizing the negative log-likelihood of the observed measurement residuals. This objective function is defined as:
% \begin{align}
%     J(\epsilon | P) = \frac{3}{2}\log 2\pi + \frac{1}{2}\log |P| + \frac{1}{2}\epsilon^{\!\top} P^{-1}\epsilon. \label{eq:nll_loss}
% \end{align}

% \begin{equation}
%     J(\epsilon \mid P) = \tfrac{3}{2}\log(2\pi) 
%     + \tfrac{1}{2}\log |R(t)| 
%     + \tfrac{1}{2}\,\epsilon(t)^{\!\top} R(t)^{-1}\epsilon(t).
%     \label{eq:nll_loss}
% \end{equation}
% The choice of this objective is motivated by the underlying state estimation framework.
While modeling the residuals with a heteroscedastic Gaussian may seem restrictive, it is a consistent assumption within the Multi-state Constraint Kalman Filter (MSCKF~\cite{mourikis_multi-state_2007, geneva_openvins_2020}), which makes Gaussian measurement noise assumptions.

In addition, temporal smoothness of the predicted covariance matrix is essential for stable state estimation and control. Rapid collapses of the covariance ellipsoid can cause abrupt changes in position estimates, leading to controller instability, particularly in high-performance autonomous racing. We quantify smoothness through the evolution of the log-determinant of $R(t)$ and impose the bound:

\begin{equation}
\label{eq:smoothness}
\frac{d}{dt} \log \det R(t) \geq -r_{\max},
\end{equation}
where $r_{\max} > 0$ is the maximum contraction rate of the ellipsoid. Even though this smoothness metric does not bound the contraction rate along all eigendirections, as the spectral norm would, it is a fast, practical and differentiable metric.

\section{Method}

As part of our method, we first infer at each timestep $t$ a symmetric positive-definite (SPD) matrix $Q$ from our core model using the robot's estimated state $\hat{x}(t)$ and the GNSS quality factors $g(t)$.
This matrix can then be utilized in two distinct ways:  
(i) directly as the desired GNSS measurement covariance, where smoothness in the dynamics is controlled by means of a loss term from Section~\ref{sec:loss}; or (ii) as a reference covariance within a Lyapunov dynamics formulation, as outlined in Section~\ref{sec:smoothness_lyap}, for which smoothness is guaranteed \textit{by design}. In this section, we discuss the model architecture of LACE, summarized in Fig.~\ref{fig:cov_arch}, the employed loss function that guarantees smoothness and finally prove the stability of our method.

\begin{figure}[!t]
    \centering
    \includegraphics[width=\columnwidth]{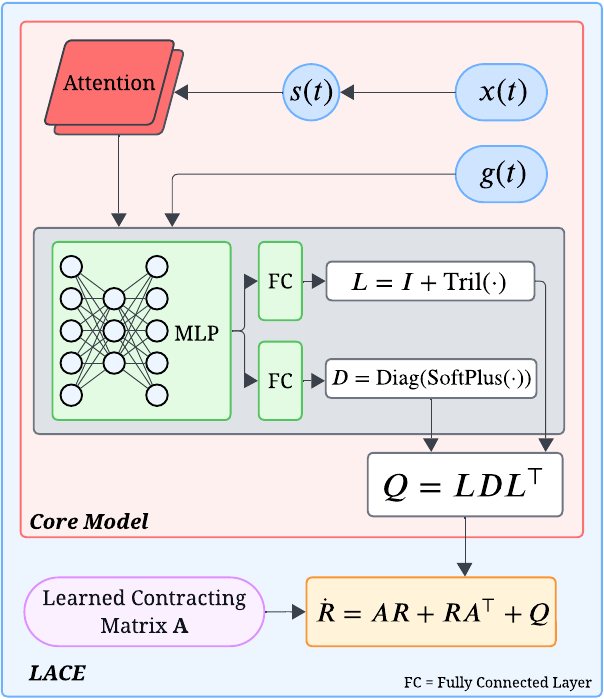}
    \caption{LACE architecture overview. The core model produces a symmetric positive-definite matrix, while during both training and inference, LACE propagates this output through its dynamical model to guarantee stability and enforce temporal smoothness.}
    \label{fig:cov_arch}
\end{figure}
A DNN is used to infer $Q$ from the estimated state of the robot $\hat{x}(t)$ and from the GNSS quality factors $g(t)$ that consist of various DOP metrics (GDOP, HDOP, VDOP, etc.) and the number of satellites used by the GNSS receiver.

\subsection{Preprocessing}

To reduce state dimensionality, $\hat{x}(t)$ is projected onto arc-length coordinates, namely the longitudinal progress $s(t)$ along the reference trajectory and its corresponding velocity $\dot{s}(t)$. This reduction is justified by the fact that the racecar repeatedly follows the same time-optimal trajectory and is later verified in Section~\ref{sec:results} by evaluating the model on multiple datasets following slightly different trajectories.

All inputs are normalized prior to being provided to the DNN. Specifically, DOP values are transformed using a logarithmic scale since they can vary over several orders of magnitude and often exhibit extreme values. The number of satellites is normalized per sequence. The longitudinal progress $s(t)$ is normalized to be between 0 and 1. Finally, the velocities $\dot{s}(t)$ are normalized, so that $\int_{0}^{T}\dot{s}(t)\,dt = T$, where $T$ is the sequence duration. This preprocessing facilitates stable optimization during training and improves the generalization performance of the network.

\subsection{Spatial Embedding and Attention Mechanism }

A global single-query attention mechanism \cite{vaswani_2017} is applied to the arc-length position $s(t)$ in order to encode spatial dependencies that may influence the output matrix. The inclusion of this mechanism is motivated by the physical environment of the racetrack, where surrounding structures such as bridges and buildings are present, and is aligned with previous work on attention-based map encoders that emphasize localized and sparse environmental features for control \cite{he2025}.

To this end, the query position is encoded into a unit-length complex query $Q = e^{i2\pi s(t)} \in \mathbb{C}$ and similarity with the key vector $K \in \mathbb{T}^a$ is given by:
\begin{align}
    \operatorname{Sim}(Q, K) := \Re(Q^*K),
\end{align}
which measures periodic alignment along the track with each key. Based on this similarity function, the attention output is computed as:
\begin{align}
    \tilde{s} \;=\;
    W \Big(
        V^\top\operatorname{Softmax}\!\left(
            \tfrac{1}{\tau}\,
            \Re(Q^* K)
        \right)
    \Big),
\end{align}
where $V \in \mathbb{R}^{a \times b}$ associates a value embedding for each key location, $\tau>0$ is the temperature parameter controlling the sharpness of the attention weights, and $W \in \mathbb{R}^{c \times b}$ is the learnable output projection. With a low temperature $\tau$, this formulation enables the attention mechanism to focus on highly localized spatial dependencies, as depicted in Figure~\ref{fig:spatial_embedding}. This attention mechanism is essential for learning the locations of sparse and localized environmental features.

\begin{figure}[!t]
    \centering
    \includegraphics[width=\columnwidth]{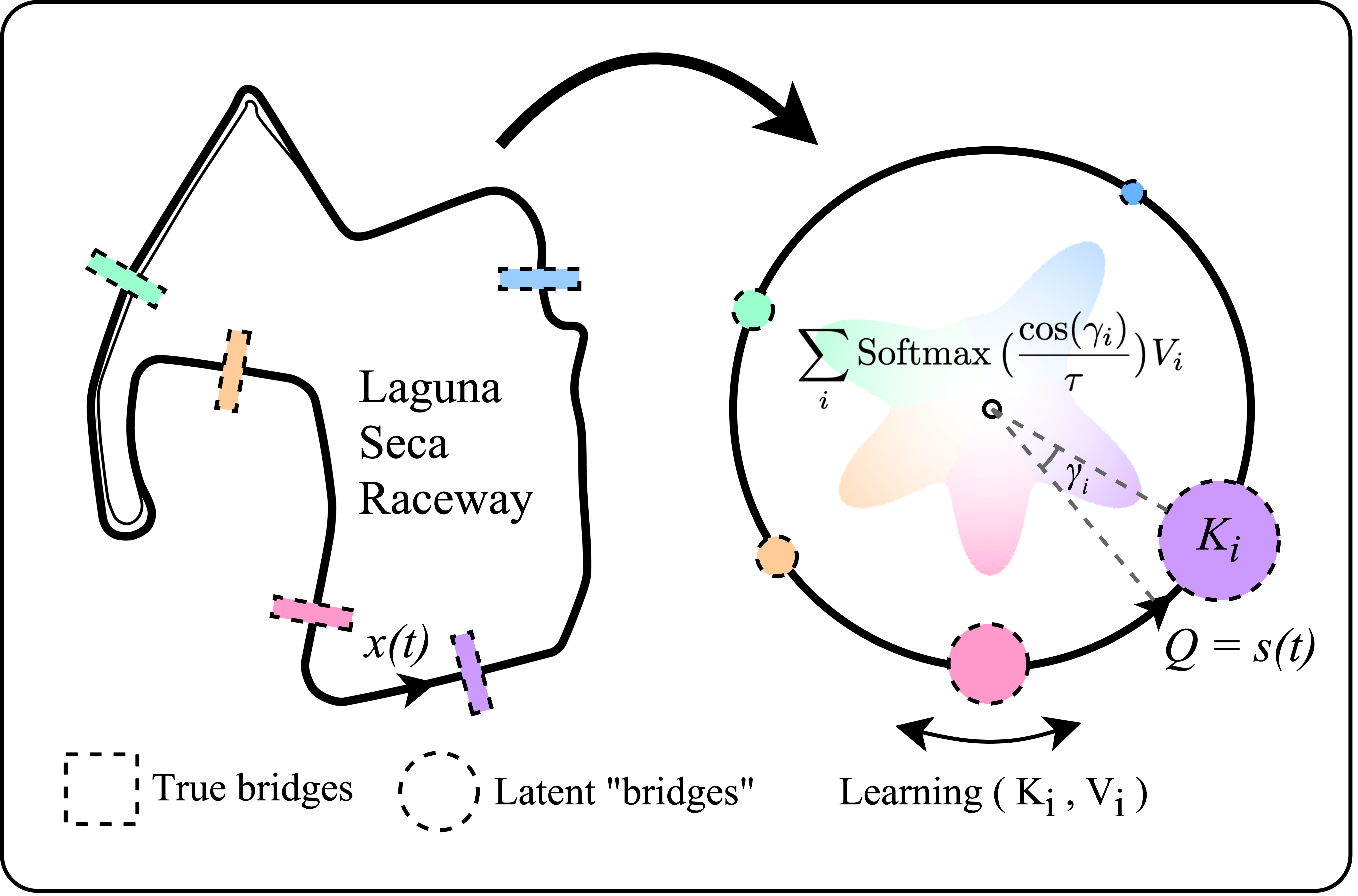}
    \caption{Illustration of the robot’s position $x(t)$, reparameterized into arc-length coordinates $s(t)$ and processed by a low-temperature attention mechanism to capture localized environmental features.}
    \label{fig:spatial_embedding}
\end{figure}

\subsection{Core Model}
\label{sec:spd_head}

A multi-layer perceptron encodes the spatial embedding $\tilde{s}(t)$, the longitudinal velocity $\dot{s}(t)$, and the GNSS quality factors $g(t)$ into a higher dimensional embedding $\phi \in \mathbb{R}^p$.

The embedding $\phi=\phi(\tilde{s}(t), \dot{s}(t), g(t))$ is finally used to produce an output positive-definite matrix $Q$. 
To ensure positive definiteness, we parameterize the process noise matrix using the LDL decomposition:
\begin{align}
    Q(\phi) = L(\phi) D(\phi) L(\phi)^{\!\top},
\end{align}
where $L(\phi) \in \mathbb{R}^{n\times n}$ is unit lower triangular and $D(\phi) \in \mathbb{R}^{n\times n}$ is diagonal with positive entries. This decomposition guarantees $Q(\phi) \succ 0$ by construction and provides a unique~\cite{higham_accuracy_2002} representation for positive definite matrices. Note that as exponential stability of $R(t)$ is maintained (Section~\ref{sec:stability}) we intentionally do \textit{not} employ spectral normalization to bound the Lipschitz constant of the DNN. This is because it is desirable for $Q(t)$ to be sensitive to $x(t)$ and $g(t)$.

% Figure~\ref{fig:ldl_decomposition} shows a visualization of the factorization.

% We parametrize the LDL factorization of the process matrix similar to \cite{hu_parametric_2015}. For $L(\phi)$ we employ a linear product weighting, and for $D(\phi)$ we use a softplus function to ensure strict positivity:
% \begin{align*}
%     D_{ii} &= \ln \!\bigl(1 + \exp(a_{i}^{\!\top} \phi)\bigr), \\[6pt]
%     L_{ij} &= 
%     \begin{cases}
%         b_{ij}^{\!\top} \phi, & i > j \\
%         1, & i = j \\
%         0, & i < j
%     \end{cases}\quad,
% \end{align*}
% where $a_i, b_{ij} \in \mathbb{R}^p$, with such parameterization, we are able to efficiently represent all positive definite matrices.

We parametrize the LDL factorization similar to~\cite{hu_parametric_2015}, with $D_{ii} = \ln(1 + \exp(a_i^\top \phi))$ ensuring strict positivity and $L_{ij} = b_{ij}^\top \phi$ for $i > j$, where $L$ is unit lower triangular and $a_i, b_{ij} \in \mathbb{R}^p$. This parameterization efficiently represents all positive definite matrices.

\subsection{Smoothness through LACE}
\label{sec:smoothness_lyap}
In the LACE framework, we model the dynamics of the measurement noise covariance using the Lyapunov differential equation:
\begin{align}
   \dot{R}(t)= A(t)R(t) + R(t)A(t)^{\!\top} + Q(t), \label{eq:cov_dynamics}
\end{align}
where $A(t)\in \mathbb{R}^{n \times n}$ is a transition matrix governing the covariance dynamics (e.g., $A(t)=\frac{\partial f(x,t)}{\partial{x}})$, $R(t)\in \mathbb{R}^{n \times n}$ represents the \textit{measurement} noise covariance matrix that we wish to estimate, rather than the state estimation covariance. In this formulation, $Q(t)\in \mathbb{R}^{n \times n}$, serves as the process noise matrix that drives the evolution of $R(t)$ and is obtained from the core model introduced in Section~\ref{sec:spd_head}. In practice, $A(t)$ can be designed to encode a desired noise dissipation process or learned as a data-driven initialization for deployment tuning (Section~\ref{sec:tunable}).

\begin{assumption}[Positive definiteness]
\label{assump:pd}
By definition, the covariance matrix must be positive semi-definite. In this work, we additionally constrain $R(t)\in\mathbb{S}^n_{++}$ for all $t \geq 0$ to allow for the well-posedness of~\eqref{eq:smoothness}. This property can be guaranteed by ensuring $Q(t) \succ 0$ and $R(0) \succ 0$.
\end{assumption}

We show that imposing a spectral constraint on $A(t)$ enforces smoothness as defined in~\eqref{eq:smoothness}.

\begin{theorem}\label{thm:spectral_smoothness}
Let $A(t)\in\mathbb{R}^{n\times n}$ be time-varying with eigenvalues $\{\lambda_i(t)\in \mathbb{C}\}_{i=1}^n$ and let $R(t)\in\mathbb{S}_{++}^n$ be differentiable. If the eigenvalues of $A(t)$ are constrained such that for all $t\geq 0$, there is a $\lambda_{\min} < 0$, such that:
\begin{equation}
\label{eq:lambda_constraint}
    \lambda_{\min} \leq \Re \big(\lambda_i(t)\big)< 0, \quad \forall i \in \{1,\dots,n\},
\end{equation}
then choosing $\lambda_{\min} \geq -\,r_{\max}/(2n)$ satisfies the smoothness constraint from~\eqref{eq:smoothness}.
\end{theorem}

\begin{proof}
Using the definition of the smoothness constraint in \eqref{eq:smoothness} and Assumption~\ref{assump:pd}, we write
\begin{equation}
\label{eq:logdet_dot}
\frac{d}{dt} \log \det R(t) = \operatorname{tr}\!\left(R(t)^{-1} \dot R(t)\right).
\end{equation}

Substituting $\dot R(t)$ in \eqref{eq:logdet_dot} with the dynamics in \eqref{eq:cov_dynamics} yields:
\begin{align*}
\frac{d}{dt} \log \det R(t) 
&= 2\operatorname{tr}\big(A(t)\big) + \operatorname{tr}\big(R(t)^{-1}Q(t)\big).
\end{align*}
Because of the positive definiteness of $R$ and $Q$, the term $\operatorname{tr}(R^{-1}Q)$ is positive, yielding the bound:
\begin{equation*}
\frac{d}{dt} \log \det R(t) \geq 2\operatorname{tr}(A(t)) =2\sum_{i=1}^n\lambda_i(t) \geq 2n\lambda_{\min}.
\end{equation*}

Substituting $\lambda_{\min}=-\,r_{\max}/(2n)$ gives the final claim.
\end{proof}

Theorem \ref{thm:spectral_smoothness} therefore allows us to bound the contraction rate of the covariance ellipsoid by imposing a spectral constraint on the eigenvalues of $A(t)$. 

\subsection{Data-efficient Implementation}

In practice, \eqref{eq:cov_dynamics} must be discretized at the GNSS sampling period $\Delta_t$, yielding the discrete-time Lyapunov recursion:
\begin{equation}
    R_{k+1} \;=\; A_{d,k}\, R_k\, A_{d,k}^{\!\top} \;+\; Q_{d,k},
    \label{eq:disc_lyapunov}
\end{equation}
where $(A_{d,k}, Q_{d,k})$ denote the discretized system and process noise matrices.

\begin{assumption}[Time-invariant diagonalizable dynamics]\label{assump:diagA}
We assume that $A$ is time-invariant and diagonalizable as $A = U \Lambda U^{-1}$, where $\Lambda = \mathrm{diag}(\lambda_1,\dots,\lambda_n)$ and $V \in \mathbb{C}^{n \times n}$ is an invertible matrix.
The eigenvalues $\{\lambda_i\}_{i=1}^n$ are learned parameters and satisfy the constraint specified in~\eqref{eq:lambda_constraint}.
\end{assumption}

Unrolling the discrete recurrence of \eqref{eq:disc_lyapunov} yields:
\begin{equation}
R_t \;=\; A_d^{\,t} R_0 \big(A_d^{\,t}\big)^{\!\top}
\;+\; \sum_{j=0}^{t-1} A_d^{\,j}\, Q_{d,\,t-1-j}\,\big(A_d^{\,j}\big)^{\!\top}.
\label{eq:cov_sum}
\end{equation}

Let $\widetilde{R}_t=U^{-1} R_t U^{-\top}$ and $\widetilde{Q}_t=U^{-1} Q_t U^{-\top}$ and ignoring the initial condition $R_0$ that contracts exponentially as shown in Section~\ref{sec:stability}, we can rewrite \eqref{eq:cov_sum} as:
\begin{equation}
    \widetilde{R}_t
    \;=\; \sum_{j=0}^{t-1} \Lambda^j \widetilde{Q}_{t-1-j} (\Lambda^j)^{\!\top},
    \label{eq:cov_sum_eig}
\end{equation}
since $\Lambda$ is diagonal, each entry of the covariance matrix evolves independently as:
\begin{equation}
    [\widetilde{R}_t]_{ij} \;=\; \sum_{k=0}^{t-1} (\lambda_i \lambda_j)^k \, [\widetilde{Q}_{t-1-k}]_{ij}.
    \label{eq:geom_series}
\end{equation}
This convolutional structure eliminates the sequential dependency in the recurrence in~\eqref{eq:disc_lyapunov}, and is well-suited for GPU hardware since it is parallelizable using the Fast Fourier Transform (FFT) with zero-padding, similar to the S4 state-space-model architecture~\cite{gu_efficiently_2022}, where a recurrent state update is reformulated as a parallelizable convolution.

\subsection{Loss Function} 
\label{sec:loss}
Given residuals $\epsilon(t)$ and inferred covariance $R(t)$, we can define the loss function as the Negative Log-Likelihood of the assumed heteroscedastic Gaussian model from \eqref{eq:dist} and incorporate an additional loss term to control the smoothness of the inferred measurement covariance, so that:
\begin{equation}
\mathcal{L}(t) = \mathcal{L}_\mathrm{NLL}(t) + \lambda\mathcal{L}_\mathrm{smooth}(t),
\end{equation}
where $\lambda > 0$ and
\begin{equation}
\label{eq:loss}
\mathcal{L}_\mathrm{NLL}(t) = \sum_{k=1}^{T}\log |R(t)| + \epsilon(t)^{\!\top} R(t)^{-1}\epsilon(t).
\end{equation}
In addition, we use an additional loss term that penalizes violations of \eqref{eq:smoothness} through a squared hinge loss:
\begin{equation}
\mathcal{L}_{\text{smooth}}(t) \;=\; 
\min\!\left(0,\, r_{\max} + \frac{d}{dt}\log \det R(t) \right)^2.
\end{equation}

Note that this penalty is inactive in the full LACE architecture since we have shown in Theorem~\ref{thm:spectral_smoothness} that $\frac{d}{dt}\log \det R(t) \geq -r_{\max}$ \textit{by design}. This penalty is nonetheless useful when comparing LACE against other baseline models for which smoothness is not guaranteed.

\subsection{Contraction Stability Analysis}
\label{sec:stability}
Consider the covariance propagation dynamics in (\ref{eq:cov_dynamics}), where $R \in \mathbb{S}^n_{++}$ is the positive definite covariance matrix, the time-varying $A(t) \in \mathbb{R}^{n \times n}$ is exponentially stable (contracting) with $\mu = -\frac{1}{2}\lambda_{\max}(A + A^{\!\top}) > 0$, and $Q(t) \in \mathbb{S}^n_{++}$ is the time-varying positive definite process noise matrix learned from environmental features.
We can also assume the learned process noise $Q_t$ is uniformly bounded: $\|Q_t\|_F \leq Q_{\max}$ for all $t \geq 0$.

\begin{theorem}
\label{thm:exp_conv}
The covariance dynamics~\eqref{eq:cov_dynamics} is incrementally exponentially stable. Any two solutions $R_1(t)$ and $R_2(t)$ with different initial conditions satisfy:
\begin{equation}
\|R_1(t) - R_2(t)\|_F \leq \|R_1(0) - R_2(0)\|_F e^{-2\mu t},
\end{equation}
where $\mu = -\frac{1}{2}\lambda_{\max}(A + A^{\!\top}) > 0$.
\end{theorem}

\begin{proof}
Let $R_1(t)$ and $R_2(t)$ be two solutions of~\eqref{eq:cov_dynamics} with different initial conditions. Define the error $\Delta R = R_1 - R_2$. The error dynamics are:
\begin{equation}
\frac{d\Delta R}{dt} = A\Delta R + \Delta RA^{\!\top}.
\end{equation}

Note that $Q_t$ cancels as it is identical for both trajectories at time $t$. Also, $\Delta R^{\!\top} = \Delta R$ (as the difference of two symmetric matrices). Consider the squared Frobenius norm $V = \|\Delta R\|_F^2 = \mathrm{tr}(\Delta R^{\!\top} \Delta R) = \mathrm{tr}(\Delta R^2)$. Taking the time derivative:
\begin{align}
\dot{V} &= 2\mathrm{tr}((A + A^{\!\top})\Delta R^2)\leq -4\mu V,
\end{align}
since $A$ is contracting with $A + A^{\!\top}$ has all negative eigenvalues.

Using the comparison lemma yields $V(t) \leq V(0)e^{-4\mu t}$. With a properly scaled bounded numerical error $\sup_t d(t)=\bar{d}$, we find an exponentially converging error bound:
\begin{equation}
\|\Delta R(t)\|_F \leq \|\Delta R(0)\|_F e^{-2\mu t}+\frac{\bar{d}(1-e^{-2\mu t})}{2\mu},
\end{equation}
which indicates the size of the error decreases with larger $\mu$ properly selected for fast convergence of $R(t)$~\cite{tsukamoto_contraction_2021}.
\end{proof}

\section{Results}
\label{sec:results}
We implemented LACE on an AV-24 autonomous racecar and conducted multiple experiments at the Laguna Seca Raceway~\cite{iac_racecar_nodate}. Our evaluation strategy leverages a high-fidelity, offline pipeline to generate ground-truth state estimates from data collected during on-track experiments. We use a LiDAR-based simultaneous localization and mapping (SLAM) offline framework to produce accurate trajectories. This pseudo-ground-truth enables a rigorous assessment of our lightweight, online GNSS covariance learning model. The following sections present our baselines, evaluation metrics, and the performance of our method integrated into the vehicle's localization stack by inferring a $R(t) \in \mathbb{R}^{3\times 3}$ GNSS covariance matrix.
\begin{figure}[!t]
    \centering
    \includegraphics[width=\columnwidth]{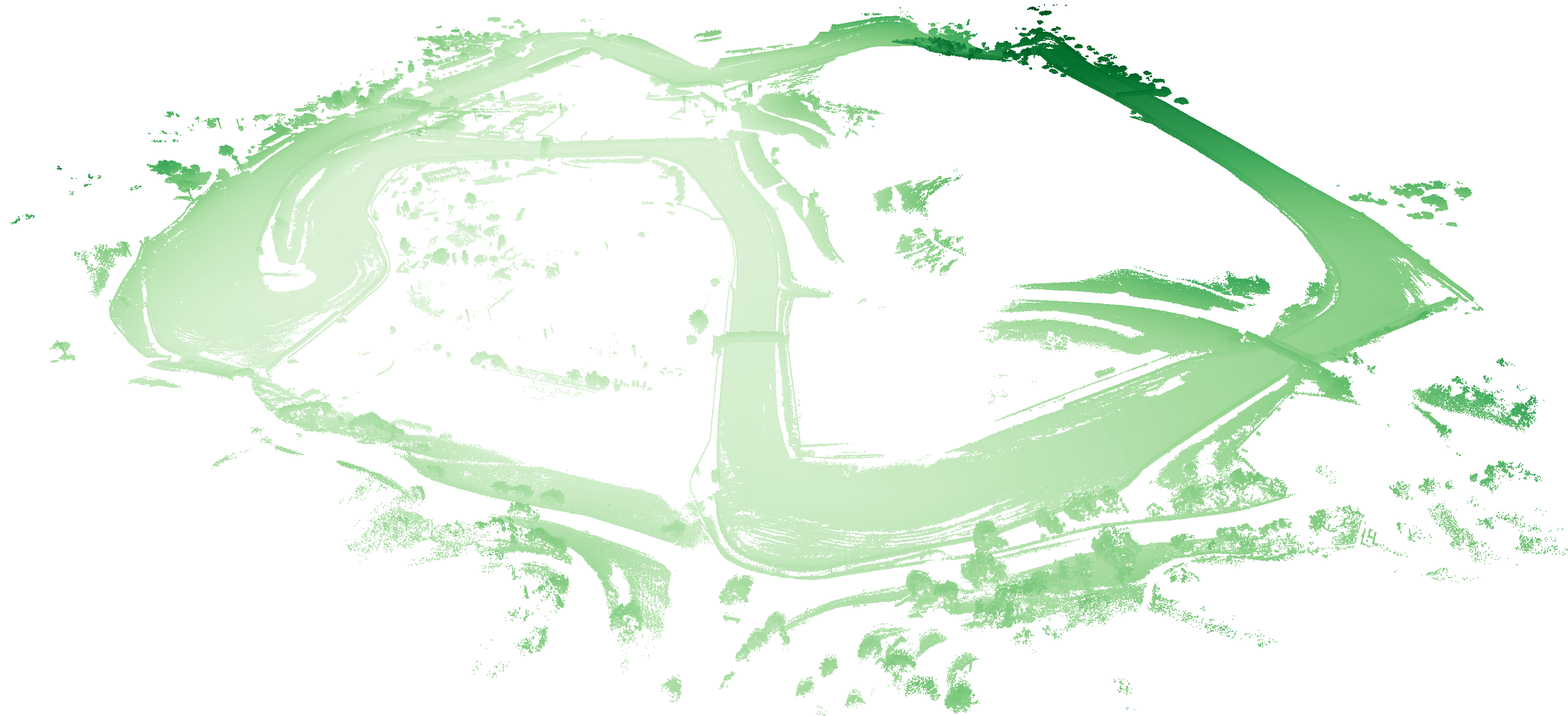}
    \caption{Dense LiDAR map of the Laguna Seca Raceway from running offline FGO based SLAM framework with color coded by altitude.}
    \label{fig:glim_map}
\end{figure}
\subsection{Datasets}

We collected sensor data from multiple practice sessions at Laguna Seca Raceway using an industrial-grade sensor suite. This setup included Luminar Iris LiDARs (20 Hz), NovAtel OEM7 GNSS receivers with RTK corrections (20 Hz), and a VectorNav VN310 IMU (200 Hz). Then, to establish a high-fidelity ground truth we employed an offline state estimation pipeline using GLIM~\cite{koide_glim_2024}, a factor-graph optimization (FGO) based SLAM framework that allows GPU-based dense LiDAR scan registration and optimal fusion with IMU and GNSS sensor measurements. To ensure robustness against signal degradation, GNSS measurements were only incorporated as factors in areas known to be free of interference with full RTK availability to reduce global drift. In open-sky segments, the SLAM trajectory agreed with RTK-fixed GNSS solutions to within 5 cm, supporting its use as a reference trajectory. The resulting dense map from this process is shown in Fig.~\ref{fig:glim_map}. By running the framework on all datasets, we obtained the pseudo-ground truth state $x(t)$ for a given run.

\begin{figure}[!t]
    \includegraphics[width=\columnwidth]{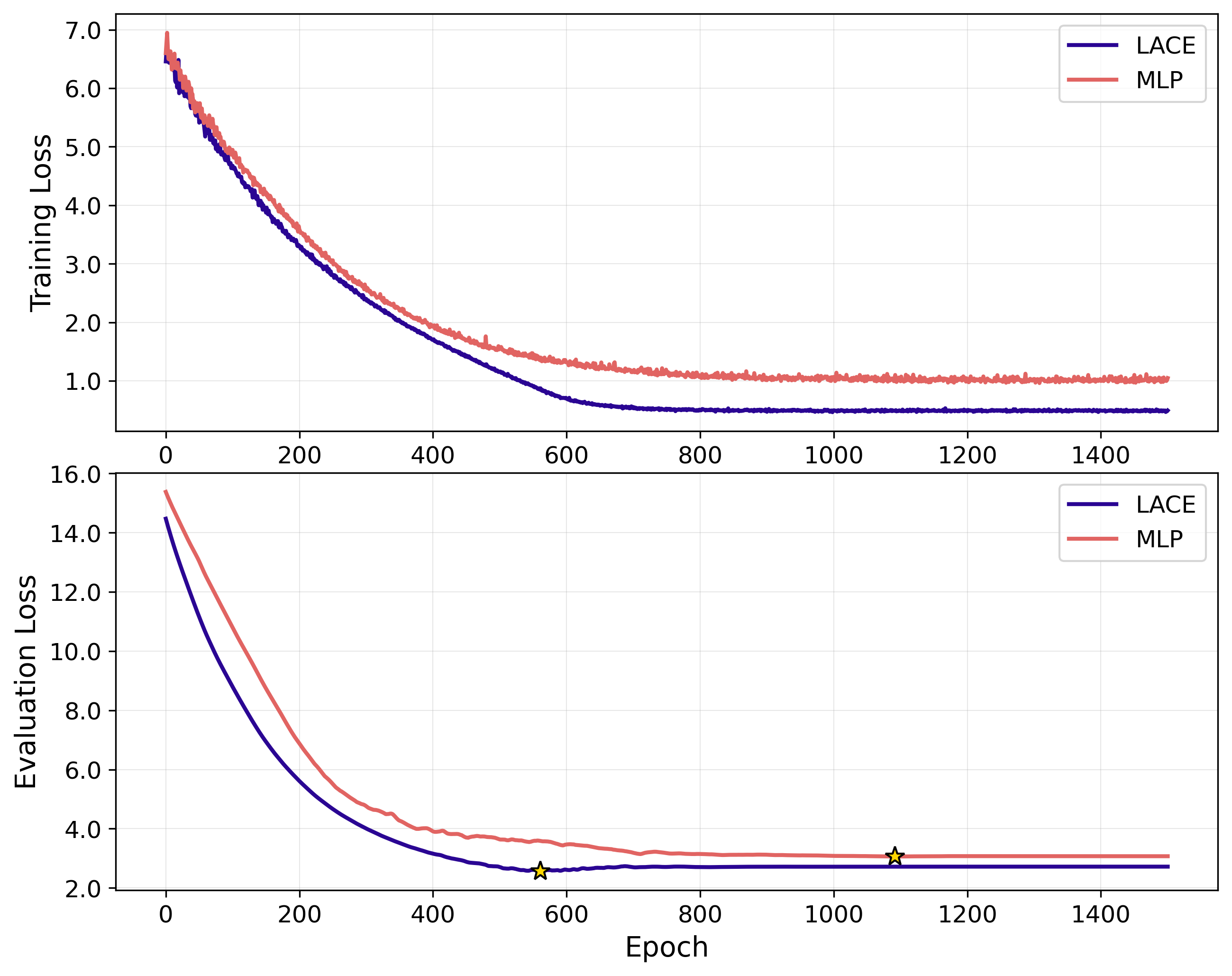}
    \caption{Train and evaluation loss of our method (LACE) and MLP, yellow stars indicate the lowest part of the evaluation loss. LACE shows a much faster training convergence rate.}
    \label{fig:loss}
\end{figure}

\subsection{Training}
To evaluate the effectiveness of our proposed dynamical framework, we compare LACE against a one-shot covariance prediction baseline using a multi-layer perceptron (MLP) that directly predicts the covariance matrix $R(t)$ at each timestep without explicitly modeling its temporal evolution. Its architecture is illustrated in Fig.~\ref{fig:cov_arch} as the \textit{Core Model} with $Q$ being the estimated covariance matrix. To ensure a fair comparison that isolates the impact of our dynamic modeling, the MLP baseline is constructed with the same architecture and number of learnable parameters as the covariance head in our LACE model. Both models were trained for a fixed number of epochs to minimize the NLL objective. As illustrated in Fig.~\ref{fig:loss}, our method not only achieves a better evaluation loss, but also converges significantly faster, demonstrating the benefits of explicitly modeling the covariance dynamics.

\begin{figure}[!t]
    \includegraphics[width=\columnwidth]{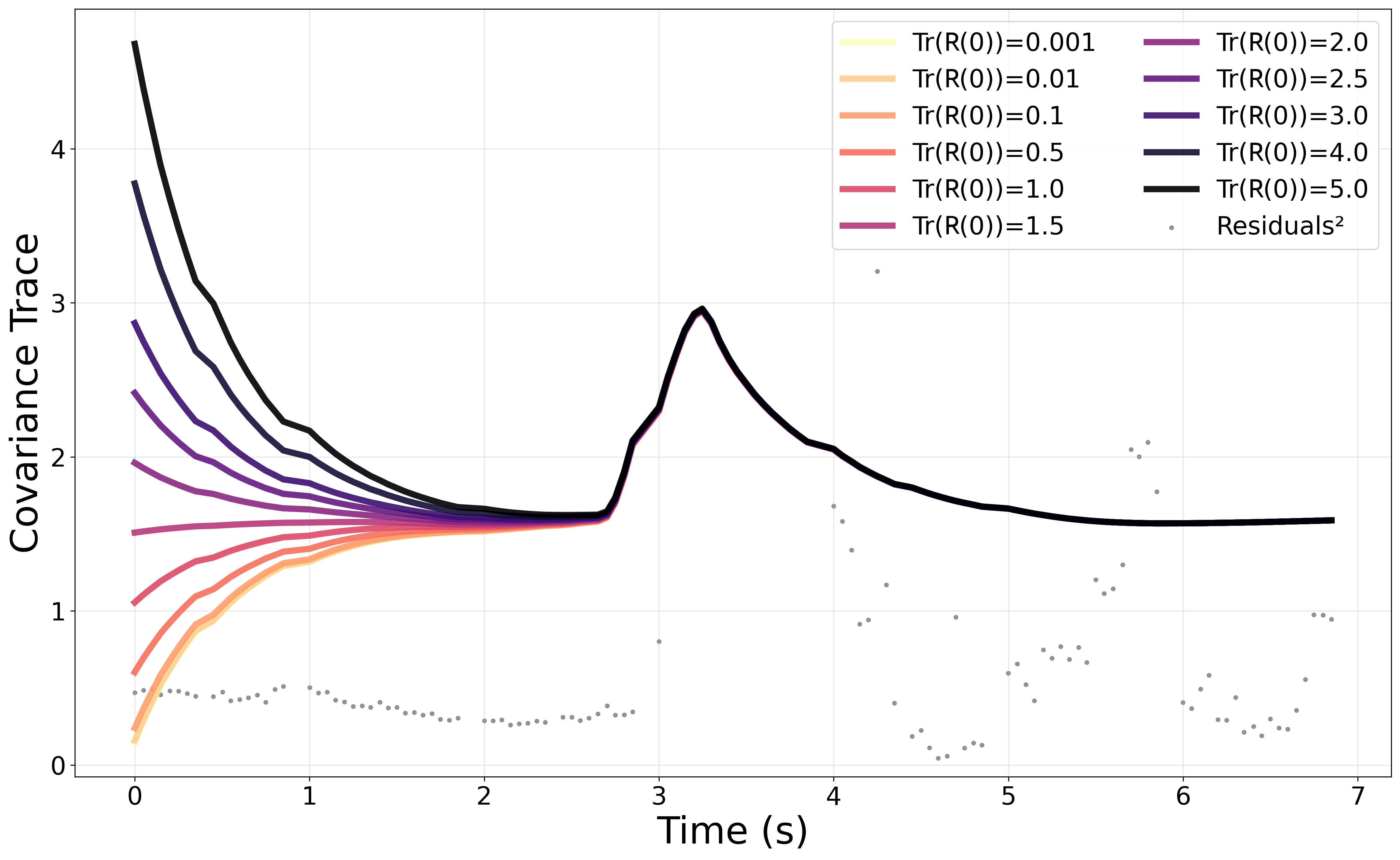}
    \caption{Guaranteed exponential convergence of covariance evolution for LACE, demonstrated by different initial covariances $R(0)$. }
    \label{fig:different_P0}
\end{figure}

\subsection{Exponential Convergence}
Our theoretical analysis in Section~\ref{sec:stability} guarantees that the covariance dynamics are incrementally exponentially stable, making our LACE framework robust to initial conditions. This property ensures that for any arbitrary positive definite initial covariance $R(0)$, the system will converge to the same covariance trajectory dictated by the learned dynamics. We validate this theoretical result experimentally, as shown in Fig.~\ref{fig:different_P0}. By initializing the system with a diverse set of covariance matrices, we demonstrate that all trajectories indeed converge exponentially to the same unique evolution, confirming the practical stability and predictability of our approach.

\begin{figure*}[!t]
    \centering
    \includegraphics[width=\textwidth]{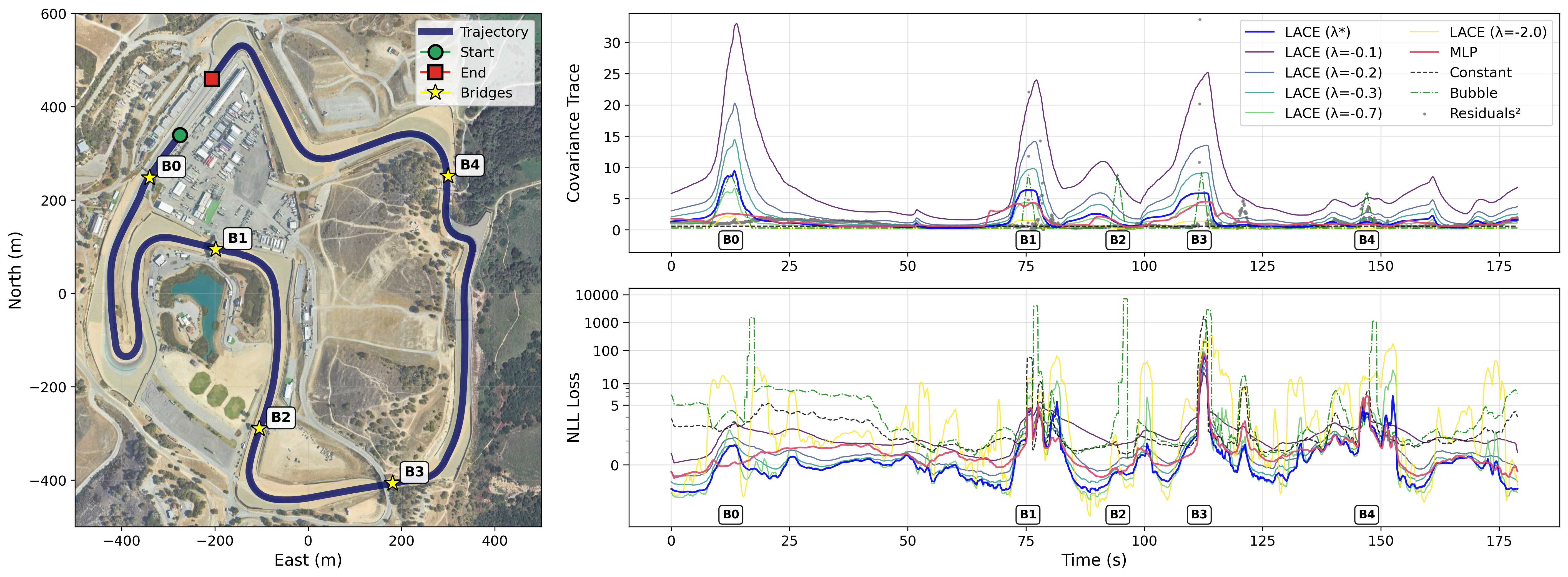}
    \caption{Covariance model comparison. \textbf{Left}: satellite view of a practice lap at Laguna Seca Raceway, with bridges that potentially degrade GNSS measurements labeled. \textbf{Top right}: traces of the covariance matrices from different models. Our method successfully encloses the residuals while maintaining smoothness. \textbf{Bottom right}: negative log likelihood of evaluated models, LACE overall has the lowest loss suggesting a better estimation of the GNSS uncertainty.}
    \label{fig:prediction}
\end{figure*}
\begin{figure}[!t]
    \includegraphics[width=\columnwidth]{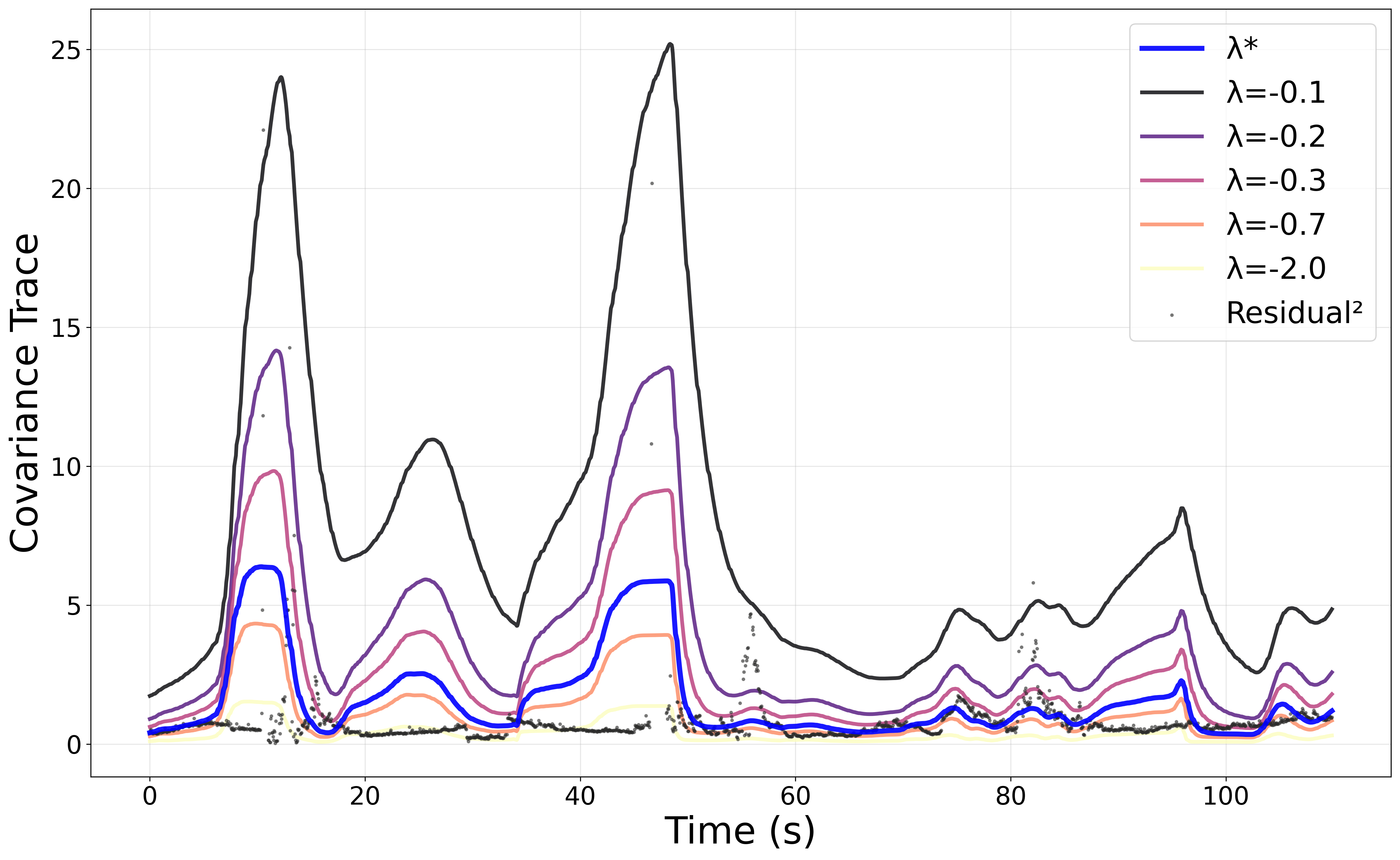}
    \caption{Effects of different choice of eigenvalues on the predicted covariance. $\lambda^*$ denotes the learned value. Lower $\lambda$ biases the learned progress noise matrix $Q(t)$ more resulting overall more conservative $R(t)$.}
    \label{fig:different_lambda}
\end{figure}
\subsection{Track Prediction}
Over the course of a lap at the Laguna Seca Raceway, our autonomous racecar encounters multiple pedestrian and vehicle bridges that severely degrade GNSS quality. To analyze the practical performance of LACE on the race track, we compared its predictions against two common empirical models in addition to the MLP baseline. The first empirical approach uses a \textit{Constant} covariance, parameterized by a scalar such that $R_{\text{constant}} = cI$ with $c \in \mathbb{R}$ and $I \in \mathbb{R}^{3 \times 3}$ is the identity matrix. The second, a heuristic \textit{Bubble} model, attempts to account for known GNSS degradation zones by linearly scaling $c$ by the distance to the preset bridge locations with some additional fixed padding. 

As illustrated in Fig.~\ref{fig:prediction}, our analysis reveals the limitations of the baseline methods relative to our solution. The \textit{Constant} model, by design, cannot capture any environmental variations in GNSS quality. While the \textit{Bubble} model provides a coarse approximation by increasing uncertainty near bridges (B0, B1, B2, and B3), its hand-tuned nature prevents it from capturing the true, time-varying signal quality. As a result, it fails to recognize the minimal degradation at another bridge (B2) during this particular lap, thereby sacrificing estimation performance. The purely adaptive MLP baseline, while responsive, produces covariance estimates that can be highly irregular subject to rapid changes. In contrast, our LACE model not only provides a better representation of the underlying residual distribution, especially in areas with erratic GDOP values, but does so while maintaining a significantly smoother covariance trajectory as a direct benefit of our explicit dynamic modeling.

\subsection{Tunable Dynamic Matrix}
\label{sec:tunable}
A key advantage of LACE is the ability to not only learn an optimal system matrix $A$ but also to manually tune its dynamics for specific operational needs. While the framework learns a set of eigenvalues, $\lambda_i^*$, that optimally minimize the loss function (Eq.~\ref{eq:loss}), these values can be manually adjusted to alter the model's sensitivity. As demonstrated in Fig.~\ref{fig:different_lambda}, this provides a powerful mechanism for balancing statistical optimality with practical robustness. Although the learned eigenvalues $\lambda^*$ yield the lowest average loss (Table~\ref{table:test_avg_loss}), they can occasionally underestimate uncertainty during moments of rapid signal degradation. By selecting a less negative eigenvalue (e.g., $\lambda = -0.1$), we can make the covariance evolution respond more aggressively to the learned process noise $Q(t)$. This results in a more conservative estimate that more consistently bounds the measurement residuals. This inherent flexibility is a significant practical benefit, allowing the system to be tuned for varying levels of environmental adversity and providing a balance between adaptivity and smoothness during deployment.

\subsection{Estimator in the Loop}
The Multisensor-aided Inertial Navigation System (MINS~\cite{lee_mins_2025}) is a MSCKF based state estimation framework that is lightweight, robust, and supports diverse sensor configurations. We adopted a modified version of MINS as part of our localization stack. To further improve real-time performance, we converted our DNN model to TensorRT (TRT), enabling efficient inference on the onboard GPU of our racecar. In this section, we present the results of our method integrated within this state-of-the-art sensor fusion framework.

We evaluated our method across different areas of the track. Figure~\ref{fig:estimator} illustrates estimator performance at Bridge 3 (B3 in Fig.~\ref{fig:prediction}), which introduces the strongest GNSS disturbances due to its size and material. LACE adapts to the highly degraded measurements and provides reliable uncertainty estimates to the EKF. In contrast, the other methods allow overconfident updates, leading to sharp deviations in state estimation which will eventually trigger unrecoverable behavior in the high-fidelity race controller. These results highlight the advantage of LACE to the downstream systems, particularly under extreme conditions.

\begin{table}[t]
\renewcommand{\arraystretch}{1.0}
\caption{Average loss across multiple laps. While $\lambda^*$ achieves the lowest average loss, a smaller $\lambda$ (such as $\lambda=-0.1$) can help robustify the model performance by biasing towards the degradation signals picked up by $Q(t)$.} \label{table:test_avg_loss}

% \caption{Average loss across multiple laps} \label{table:test_avg_loss}
\centering
\begin{tabular}{lcc}
\toprule
Model & Avg Loss & Std \\
\midrule
LACE ($\lambda^*$)   & \textbf{2.9694} & 5.3995 \\
LACE ($\lambda=-0.1$) & 5.3642 & \textbf{0.5913} \\
LACE ($\lambda=-0.2$) & 4.4280 & 1.1541 \\
LACE ($\lambda=-0.3$) & 3.9281 & 1.7011 \\
LACE ($\lambda=-0.7$) & 3.1205 & 3.8769 \\
LACE ($\lambda=-2.0$) & 3.2145 & 11.8982 \\
MLP                   & 3.4167 & 3.2236 \\
Constant              & 5.8491 & 33.0986 \\
Bubble                & 4.2820 & 14.8187 \\
\bottomrule
\end{tabular}
\end{table}
\begin{figure}[!t]
    \includegraphics[width=\columnwidth]{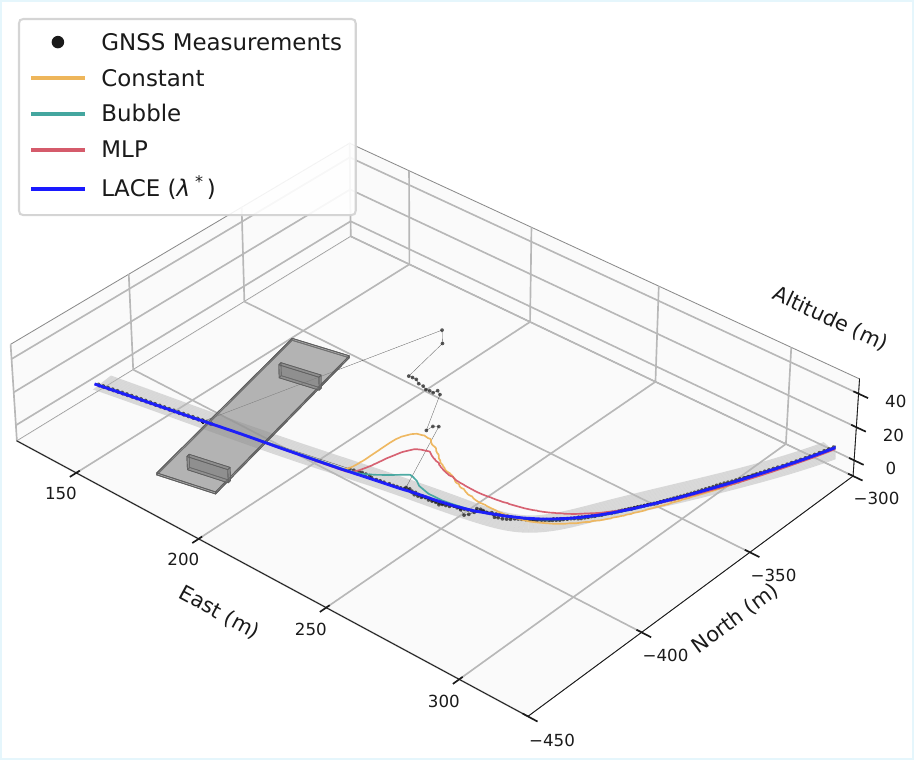}
    \caption{Trajectories estimated from MINS with different covariance models. As the racecar passes under the bridge (illustrated by the gray slabs), the GNSS quality becomes strongly disturbed, as indicated by the spread of measurements. Our method, with an optimal $\lambda$, successfully captures the rapid changes in noise distribution, whereas other methods produce overconfident covariance estimates that mislead the estimator.}
    \label{fig:estimator}
\end{figure}
\section{Conclusion}
We presented LACE, a method for estimating and adapting GNSS measurement covariance for high-speed autonomous systems. In particular, our method achieves superior environmental adaptivity while guaranteeing temporal smoothness for the downstream feedback controller. By modeling covariance evolution as a provably stable dynamical system with spectral constraints, our LACE framework provides a novel method for learning environment-aware uncertainty. We also derived formal guarantees of exponential stability and smoothness bound for the covariance dynamics, ensuring that the resulting estimates are smooth by design. Through validation with data from an autonomous racecar in a real-world racing environment, we demonstrated the effectiveness of our framework in producing smoother and more accurate covariance predictions in challenging, GNSS-degraded conditions.

For future work, the extension of our dynamical modeling approach to other sensing modalities, such as LiDAR and vision, could provide a unified framework for multi-sensor covariance estimation. Furthermore, moving beyond Gaussian noise assumptions to handle arbitrary distributions (such as using flow models) would further broaden its capability to diverse sensing configurations with higher granularity. Finally, exploring the explicit co-design of the covariance model and the downstream vehicle controller could lead to new paradigms in integrated state estimation and control, further enhancing the stability and performance of safety-critical autonomous systems.

\section{Acknowledgment}
The authors thank R. Blanchard, M. Anderson, X. Zuo, P. Spieler, J. Lathrop, J. Cho, N. Ranganathan, and the other members of the Caltech Racer team.
% \section*{Acknowledgment}
% Omitted for anonymous review.

\bibliographystyle{IEEEtran}
\bibliography{IEEEabrv, references_abrv}
\end{document}